\pdfoutput=1
\documentclass[10pt,a4paper]{article}
\usepackage{algorithm}
\usepackage{algorithmic}
\usepackage{amssymb}
\usepackage{amsfonts}
\usepackage{amsmath,amsthm}
\usepackage[algoruled,linesnumbered,algo2e,vlined]{algorithm2e}

\title{Online linear optimization with the log-determinant regularizer}
\author{%
 Ken-ichiro MORIDOMI \thanks{Department of Informatics, Kyushu University} \and
 Kohei HATANO\thanks{Department of Informatics, Kyushu University} \and
 Eiji TAKIMOTO\thanks{Department of Informatics, Kyushu University}
}

\begin{document}
\date{}
\maketitle

\newtheorem{defi}{Definition}[section]
\newtheorem{theo}{Theorem}[section]
\newtheorem{prop}{Proposition}[section]
\newtheorem{coro}{Corollary}[section]
\newtheorem{assu}{Assumption}[section]
\newtheorem{lemm}{Lemma}[section]

\newtheorem{fact}{Fact}
\newtheorem{ex}{Example}

\def\remark{\par\noindent\hangindent0pt{\bf Remark.}~}

\def\OMIT#1{}
\def\newwd#1{{\em #1}}

\newcommand{\mnote}[1]{\marginpar{#1}}
\newcommand{\mynote}[1]{{\bf {#1}}}

\newcommand{\bds}[1]{\boldsymbol{#1}}
\renewcommand{\vec}[1]{\boldsymbol{#1}}
\newcommand{\mat}[1]{\mathrm{#1}}
\newcommand{\msf}[1]{\mathsf{#1}}
\newcommand{\mbf}[1]{\mathrm{#1}}
\newcommand{\mbb}[1]{\mathbb{#1}}
\newcommand{\mcal}[1]{\mathcal{#1}}
\newcommand{\mfr}[1]{\mathfrak{#1}}
\newcommand{\T}{\msf{T}}
\newcommand{\rd}{\partial}
\newcommand{\I}{\infty}

\newcommand{\Ind}{\mathbf{1}}
\newcommand{\sceq}{\succeq}
\newcommand{\Tr}{\mathrm{Tr}}
\renewcommand{\det}{\mathrm{det}}
\newcommand{\sym}{\mathrm{sym}}
\newcommand{\rank}{\mathrm{rank}}
\newcommand{\sgn}{\mathrm{sgn}}
\newcommand{\diag}{\mathrm{diag}}
\newcommand{\conv}{\mathrm{conv}}
\newcommand{\vectorize}{\mathrm{vec}}

\newcommand{\kp}{\otimes}
\newcommand{\ip}{\cdot}
\newcommand{\fp}{\bullet}
\newcommand{\cp}{\circ}

\newcommand{\diff}[2]{\frac{\rd #1}{\rd #2}}
\newcommand{\hess}[3]{\frac{\rd^2 #1}{\rd #2 \rd #3}}

\newcommand{\n}[1]{\| #1 \|}
\newcommand{\nOp}[1]{\| #1 \|_\mathrm{Op}}
\newcommand{\nSp}[1]{\| #1 \|_\mathrm{Sp}}
\newcommand{\nFr}[1]{\| #1 \|_\mathrm{Fr}}
\newcommand{\nTr}[1]{\| #1 \|_\mathrm{Tr}}
\newcommand{\nMax}[1]{\| #1 \|_\mathrm{max}}

\newcommand{\matZero}{\mbf{0}}
\newcommand{\matA}{\mbf{A}}
\newcommand{\matB}{\mbf{B}}
\newcommand{\matC}{\mbf{C}}
\newcommand{\matD}{\mbf{D}}
\newcommand{\matE}{\mbf{E}}
\newcommand{\matF}{\mbf{F}}
\newcommand{\matG}{\mbf{G}}
\newcommand{\matH}{\mbf{H}}
\newcommand{\matI}{\mbf{I}}
\newcommand{\matJ}{\mbf{J}}
\newcommand{\matK}{\mbf{K}}
\newcommand{\matL}{\mbf{L}}
\newcommand{\matM}{\mbf{M}}
\newcommand{\matN}{\mbf{N}}
\newcommand{\matO}{\mbf{O}}
\newcommand{\matP}{\mbf{P}}
\newcommand{\matQ}{\mbf{Q}}
\newcommand{\matR}{\mbf{R}}
\newcommand{\matS}{\mbf{S}}
\newcommand{\matT}{\mbf{T}}
\newcommand{\matU}{\mbf{U}}
\newcommand{\matV}{\mbf{V}}
\newcommand{\matW}{\mbf{W}}
\newcommand{\matX}{\mbf{X}}
\newcommand{\matY}{\mbf{Y}}
\newcommand{\matZ}{\mbf{Z}}
\newcommand{\matSigma}{\mbf{\Sigma}}
\newcommand{\matTheta}{\mbf{\Theta}}
\newcommand{\matPhi}{\mbf{\Phi}}
\newcommand{\matPi}{\mbf{\Pi}}

\newcommand{\hmatX}{\hat{\matX}}
\newcommand{\hmatY}{\hat{\matY}}
\newcommand{\hmatZ}{\hat{\matZ}}
\newcommand{\hmatPi}{\hat{\matPi}}

\newcommand{\tmatY}{\tilde{\matY}}
\newcommand{\tmatPi}{\tilde{\matPi}}

\newcommand{\veczero}{\bds{0}}
\newcommand{\veca}{\bds{a}}
\newcommand{\vecb}{\bds{b}}
\newcommand{\vecc}{\bds{c}}
\newcommand{\vecd}{\bds{d}}
\newcommand{\vece}{\bds{e}}
\newcommand{\vecf}{\bds{f}}
\newcommand{\vecg}{\bds{g}}
\newcommand{\vech}{\bds{h}}
\newcommand{\veci}{\bds{i}}
\newcommand{\vecj}{\bds{j}}
\newcommand{\veck}{\bds{k}}
\newcommand{\vecl}{\bds{\ell}}
\newcommand{\vecm}{\bds{m}}
\newcommand{\vecn}{\bds{n}}
\newcommand{\veco}{\bds{o}}
\newcommand{\vecp}{\bds{p}}
\newcommand{\vecq}{\bds{q}}
\newcommand{\vecr}{\bds{r}}
\newcommand{\vecs}{\bds{s}}
\newcommand{\vect}{\bds{t}}
\newcommand{\vecu}{\bds{u}}
\newcommand{\vecv}{\bds{v}}
\newcommand{\vecw}{\bds{w}}
\newcommand{\vecx}{\bds{x}}
\newcommand{\vecy}{\bds{y}}
\newcommand{\vecz}{\bds{z}}
\newcommand{\vecpi}{\bds{\pi}}
\newcommand{\vecone}{\bds{1}}
\newcommand{\vecsigma}{\bds{\sigma}}
\newcommand{\veczeta}{\bds{\zeta}}

\newcommand{\hvecx}{\hat{\vecx}}
\newcommand{\hvecy}{\hat{\vecy}}
\newcommand{\hvecz}{\hat{\vecz}}
\newcommand{\hvecpi}{\hat{\vecpi}}

\newcommand{\calA}{\mcal{A}}
\newcommand{\calB}{\mcal{B}}
\newcommand{\calC}{\mcal{C}}
\newcommand{\calD}{\mcal{D}}
\newcommand{\calE}{\mcal{E}}
\newcommand{\calF}{\mcal{F}}
\newcommand{\calG}{\mcal{G}}
\newcommand{\calH}{\mcal{H}}
\newcommand{\calI}{\mcal{I}}
\newcommand{\calJ}{\mcal{J}}
\newcommand{\calK}{\mcal{K}}
\newcommand{\calL}{\mcal{L}}
\newcommand{\calM}{\mcal{M}}
\newcommand{\calN}{\mcal{N}}
\newcommand{\calO}{\mcal{O}}
\newcommand{\calP}{\mcal{P}}
\newcommand{\calQ}{\mcal{Q}}
\newcommand{\calR}{\mcal{R}}
\newcommand{\calS}{\mcal{S}}
\newcommand{\calT}{\mcal{T}}
\newcommand{\calU}{\mcal{U}}
\newcommand{\calV}{\mcal{V}}
\newcommand{\calW}{\mcal{W}}
\newcommand{\calX}{\mcal{X}}
\newcommand{\calY}{\mcal{Y}}
\newcommand{\calZ}{\mcal{Z}}

\newcommand{\realR}{\mbb{R}}
\newcommand{\realS}{\mbb{S}}

\newcommand{\expE}{\mbb{E}}

\newcommand{\eps}{\epsilon}

\newcommand{\Ra}{\mfr{R}}
\newcommand{\ERa}{\hat{\Ra}}
\newcommand{\sample}{\calS}
\newcommand{\Kt}{K_\tau}

\newcommand{\loss}{\ell}
\newcommand{\Eloss}{\hat{\loss}}

\newcommand{\Lipschitz}{L}

\def\tt#1{{#1 \times #1}}
\def\Tsum{\sum_{t=1}^{T}}
\def\l{\vec{\ell}}
\def\realRN{\mbb{R}^N}
\def\realRmn{\mbb{R}^{m \times n}}
\def\realRnm{\mbb{R}^{n \times m}}
\def\realSNN{\mbb{S}^\tt{N}}
\def\realSNNPSD{ \realSNN_+ }

\def\mKvn{\{ \x \in \mR^{N}_{\leq \mbf{0}} : \n{\x}_1 \leq \tau \}}
\def\mKbg{\{ \x \in \mR^{N}_{\leq \mbf{0}} : \n{\x}_\I \leq \sigma \}}
\def\mKgd{\{ \x \in \mR^{N}_{\leq \mbf{0}} : \n{\x}_2 \leq \rho \}}

\def\mKVN{\{ \X \in \mS^{\tt{N}}_{\sc \mbf{0}} : \nKF{\X} \leq \tau \}}
\def\mKBG{\{ \X \in \mS^{\tt{N}}_{\sc \mbf{0}} : \nOP{\X} \leq \sigma \}}
\def\mKGD{\{ \X \in \mS^{\tt{N}}_{\sc \mbf{0}} : \nF{\X} \leq \rho \}}

\def\mLvn{\{ \l \in \mR^N : \n{\l}_\I \leq \gamma_\I \}}
\def\mLbg{\{ \l \in \mR^N : \n{\l}_1 \leq \gamma_1 \}}
\def\mLgd{\{ \l \in \mR^N : \n{\l}_2 \leq \gamma_2 \}}

\def\mLVN{\{ \L \in \mS^{\tt{N}} : \nOP{\L} \leq \gamma_\I \}}
\def\mLBG{\{ \L \in \mS^{\tt{N}} : \nKF{\L} \leq \gamma_1 \}}
\def\mLGD{\{ \L \in \mS^{\tt{N}} : \nF{\L} \leq \gamma_2 \}}

\def\SymPSD{ \mS^{\tt{N}}_{\sc \mbf{0}} }
\def\Sym{ \mS^{\tt{N}} }

\newcommand{\OLO}{\mathrm{OLO}}
\newcommand{\BT}{$(\beta,\tau)$}

\begin{abstract}
We consider online linear optimization over
symmetric positive semi-definite matrices, which has various
applications including the online collaborative filtering.
The problem is formulated as a repeated game between the algorithm
and the adversary, where in each round $t$ the algorithm and the adversary
choose matrices $\matX_t$ and $\matL_t$, respectively, and
then the algorithm suffers a loss given by the Frobenius
inner product of $\matX_t$ and $\matL_t$.
The goal of the algorithm is to minimize the cumulative loss.
We can employ a standard framework called Follow the Regularized Leader
(FTRL) for designing algorithms, where we need to choose an appropriate
regularization function to obtain a good performance guarantee.
We show that the log-determinant regularization works better
than other popular regularization functions
in the case where the loss matrices $\matL_t$ are all sparse.
Using this property, we show that our algorithm achieves an
optimal performance guarantee for the online collaborative filtering.
The technical contribution of the paper is to develop a new 
technique of deriving performance bounds by exploiting the
property of strong convexity of the log-determinant with respect to
the loss matrices, while in the previous analysis the strong convexity
is defined with respect to a norm.
Intuitively, skipping the norm analysis results in the improved bound.
Moreover, we apply our method to online linear optimization over vectors
and show that the FTRL with the Burg entropy regularizer,
which is the analogue of the log-determinant regularizer
in the vector case, works well.

\end{abstract}

\section{Introduction}

Online predicion is a theoretical model of repeated processes
of making decisions and receiving feedbacks, and has been
extensively studied in the machine learning community
for a couple of decades~\cite{Degenne2016,Kale2016,Cutkosky2017}.
Typically, decisions are formulated as vectors in a fixed
set called the decision space and feedbacks as functions
that define the losses for all decision vectors.
Recently, much attention has been paid to a more general
setting where decisions are formulated as matrices, since
it is more natural for some applications such as
ranking and recommendation tasks~\cite{Cesa-Bianchi2011,Herbster2016,Jin2016}.

Take the online collaborative filtering as an example.
The problem is formulated as in the following protocol:
Assume we have a fixed set of $n$ users and a fixed set of $m$ items.
For each round $t=1,2,\dots,T$, the following happens.
(\romannumeral1) The algorithm receives from the environment
a user-item pair $(i_t,j_t)$,
(\romannumeral2) the algorithm predicts how much user $i_t$ likes
item $j_t$ and chooses a number $x_t$ that represents
the degree of preference,
(\romannumeral3) the environment returns the true evaluation value $y_t$
of the user $i_t$ for the item $j_t$, and then
(\romannumeral4) the algorithm suffers loss defined by the
prediction value $x_t$ and the true value $y_t$, say, $(x_t - y_t)^2$.
Note that, (\romannumeral3) and (\romannumeral4) in the
protocol above can be generalized in the following way:
(\romannumeral3) the environment returns a loss function $\loss_t$, say
$\loss_t(x) = (x - y_t)^2$, and (\romannumeral4) the algorithm suffers
loss $\loss_t(x_t)$.
The goal of the algorithm is to minimize the cumulative loss, 
or more formally, to minimize the \emph{regret}, which is the
most standard measure in online prediction.
The regret is the difference between the cumulative loss of
the algorithm and that of the best fixed prediction policy
in some policy class. Note that the best policy is determined
in hindsight, i.e., it depends on the whole feedback sequence.
Now we claim that the problem above can be regarded
as a matrix prediction problem:
the algorithm chooses (before observing the pair $(i_t,j_t)$)
the prediction values for all pairs as an $n \times m$ matrix,
although only the $(i_t,j_t)$-th entry is used as the prediction.
In this perspective, the policy class is formulated as a restricted
set of matrices, say, the set of matrices of bounded trace norm,
which is commonly used in collaborative filtering~\cite{Srebro2005,Mazumder2010,Shamir2011,Koltchinskii2011}.
Moreover, we can assume without loss of generality that
the prediction matrices are also chosen from the policy class.
So, the policy class is often called the decision space.

In most application problems including the online collaborative filtering,
the matrices to be predicted are not square, which makes the analysis difficult. 
Hazan et.al.~\cite{HazanKaleShalev-Shwartz2012} show that
any online matrix prediction problem formulated as in the protocol above
can be reduced to an online prediction problem where the decision
space consists of symmetric positive semi-definite matrices
under linear loss functions. A notable property of the reduction
is that the loss functions of the reduced problem are the inner
product with sparse loss matrices, where only at most 4 entries
are non-zero. 
Thus, we can focus on the online prediction problems
for symmetric positive semi-definite matrices,
which we call the online semi-definite programming (online SDP) problems.
In particular we are interested in the case where the problems
are obtained by the reduction, which we call
the online \emph{sparse} SDP problems.
Thanks to the symmetry and positive semi-definiteness of the
decision matrices and the sparseness of the loss matrices, the problem
becomes feasible and Hazan et.al.\ propose an algorithm for the
online sparse SDP problems, by which they give regret bounds
for various application problems including the online max-cut,
online gambling, and the online collaborative
filtering~\cite{HazanKaleShalev-Shwartz2012}.
Unfortunately, however, all these bounds turn out to be sub-optimal.

In this paper, we propose an algorithm for the online sparse SDP problems
by which we achieve optimal regret bounds for those application
problems.

To this end, we employ a standard framework called
Follow the Regularized Leader (FTRL) for designing and anlyzing
algorithms~\cite{Hazan2009,RakhlinAbernethyAgarwalBartlettHazanTewari2009,Shalev-Shwartz2012,Haz2016}, 
where we need to choose as a parameter an appropriate regularization function (or regularizer) to obtain a good
regret bound.
Hazan et al.~use the von Neumann entropy (or sometimes called
the matrix negative entropy) as the regularizer to obtain
the results mentioned above\cite{HazanKaleShalev-Shwartz2012},
which is a generalization of Tsuda et al.~\cite{TsudaRatschWarmuth2005}.
Another possible choice is the log-determinant regularizer,
whose Bregmann divergence is so called the LogDet divergence.
There are many applications of the LogDet divergence such as
metric learning \cite{DavisKulisJainSraDhillon2007} 
and Gaussian graphical models \cite{RavikumarWainwrightRaskuttiYu2011}.
However, the log-determinant regularizer is less popular in
online prediction and it is unclear how to derive general and
non-trivial regret bounds when using the FTRL with the
log-determinant regularizer, as posed as an open problem
in~\cite{TsudaRatschWarmuth2005}.
Indeed, Davis et al.\ apply the FTRL
with the log-determinant regularizer for square loss and give a
cumulative loss bound \cite{DavisKulisJainSraDhillon2007},
but it contains a data-dependent parameter and the regret bound is
still unclear.
Christiano considers a very specific sub-class of online sparse SDP
problems and succeeds to improve the regret bound for a particular
application problem, the online max-cut problem~\cite{Christiano2014}. 
But the problems he examines do not cover the whole class of
online sparse SDP problems and hence his algorithm cannot be applied to
the online collaborative filtering, for instance.

In this paper, we improve regret bounds for online sparse
SDP problems by analyzing the FTRL with the log-determinant regularizer.
In particular, our contributions are summarized as follows.
\begin{enumerate}
\item
We give a non-trivial regret bound of the FTRL
with the log-determinant regularizer for a general class of
online SDP problems. Although the bound seems to be somewhat loose,
it gives a tight bound when the matrices are diagonal
(which corresponds to the vector predictions).

\item
We extend the analysis of Christiano in
\cite{Christiano2014} and develop a new technique of deriving
regret bounds by exploiting the property of strong convexity
of the regularizer with respect to the loss matrices.
The analysis in~\cite{Christiano2014} is not explicitly
stated as in a general form and focused on a very specific
case where the loss matrices are block-wise sparse.

\item
We improve the regret bound for the online sparse SDP problems,
by which we give optimal regret bounds for the application problems,
namely, the online max-cut, online gambling, and the online
collaborative filtering.

\item
We apply the results to the case where the decision space consists
of vectors, which can be reduced to online matrix prediction problems
where the decision space consists of diagonal matrices.
In this case, the general regret bound mentioned in 1 also gives
tight regret bound.
\end{enumerate}

\section{Problem setting}

We first give the notations and then describe the problem
setting: the online semi-definite programming problem
(online SDP problem, for short).

\subsection{Notations}

Throughout the paper, a roman capital letter indicates a matrix. 
Let $\realRmn$, $\realSNN$, $\realSNNPSD$ 
denote the set of $m \times n$ matrices, 
the set of $\tt{N}$ symmetric matrices, 
and the set of $\tt{N}$ symmetric positive semi-definite matrices, 
respectively. 

We write the trace of a matrix $\matX$ as $\Tr(\matX)$ and the
determinant as $\det(\matX)$. 
We write the trace norm of $\matX$ as $\nTr{\matX} = \sum_{i} \sigma_i $,
the spectral norm as $\nSp{\matX} = \max_i \sigma_i $,
and the Frobenius norm as $\nFr{\matX} = \sqrt{\sum_i \sigma_i^2}$,
where $\sigma_i$ is the $i$-th largest singular value of $\matX$.
Note that if $\matX$ is positive semi-definite, then
$\Tr(\matX) = \nTr{\matX}$ and $\sigma_i$ is the
$i$-th largest eigenvalue of $\matX$.
The identity matrix is denoted by $\matE$. 
For any positive integer $m$, we write $[m] = \{1, 2, \dots m\}$. 
We define the vectorization of a matrix $\matX \in \realR^{m \times n}$ as
$$\vectorize(\matX) = (X_{*,1}^\T ,X_{*,2}^\T ,\dots,X_{*,m}^\T )^\T ,$$
where $X_{*,i}$ is the $i$-th column of $\matX$. 
For a vector $\vecx \in \realR^N$.
$\diag(\vecx)$ denote the $\tt{N}$ diagonal matrix $\matX$ such that
$X_{i,i} = x_i$.
For $m \times n$ matrices $\matX$ and $\matL$, 
$ \matX \fp \matL = \sum_{i,j}^{m,n} X_{i,j}L_{i,j}
= \vectorize(\matX)^\T \vectorize(\matL) $ is the Frobenius inner product. 

For a differentiable function $R : \realRmn \to \realR$,
its gradient $\nabla R(\matX)$ is the $m \times n$ matrix 
whose $(i,j)$-th componet is
$\diff{R(\matX)}{X_{i,j}}$, and its Hessian
$\nabla^2 R(\matX)$ is the $mn \times mn$ matrix
whose $((i,j), (k,l))$-th component is
$\hess{R(\matX)}{X_{i,j}}{X_{k,l}}$~\cite{Dattorro2005}. 

\subsection{Online SDP problem}

We consider an online linear optimization problem
over symmetric semi-definite matrices, which we call
the online SDP problem.
The problem is specified by a pair $(\calK, \calL)$, where
$\calK \subseteq \realSNNPSD$ is a convex set of symmetric
positive semi-definite matrices and $\calL \subseteq \realSNN$
is a set of symmetric matrices.
The set $\calK$ is called the decision space and $\calL$ the loss space.
The online SDP problem $(\calK, \calL)$ is a repeated game
between the algorithm and the adversary (i.e., an environment
that may behave adversarially), which is described as the
following protocol.

In each round $t = 1, 2, \ldots, T$, the algorithm
\begin{enumerate}
  \item chooses a matrix $\matX_t \in \calK$,
  \item receives a loss matrix $\matL_t \in \calL$ from the adversary, and
  \item suffers the loss $\matX_t \fp \matL_t$.
\end{enumerate}
The goal of the algorithm is to minimize the regret
$Reg(T,\calK,\calL)$, defined as
\[
	Reg(T,\calK,\calL) = \Tsum \matL_t \fp \matX_t
		- \Tsum \matL_t \fp \matU,
\]
where $\matU = \arg\min_{\matX \in \calK} \Tsum \matL_t \fp \matX$ is
the best matrix in the decision set $\calK$ that minimizes
the cumulative loss. The matrix $\matU$ is called the
best offline matrix. 

\subsection{Online linear optimization over vectors}

The online SDP problem is a generalization of the
online linear optimization problem over vectors, which is a more standard
problem setting in the literature.
For the ``vector'' case, the problem is described as the
following protocol:

In each round $t = 1, \cdots, T$, the algorithm
\begin{enumerate}
  \item chooses $\vecx_t \in \calK \subset \realR^N_+$,
  \item receives $\vecl_t \in \calL \subset \realR^N$ from the adversary, and
  \item suffers the loss $\vecx_t^\T \vecl_t$.
\end{enumerate}
It is easy to see that
the problem is equivalent to the online SDP problem $(\calK' \calL')$
where $\calK' = \{\diag(\vecx) \mid \vecx \in \calK\}$
and $\calL' = \{\diag(\vecl) \mid \vecl \in \calL\}$.
So, all the results for the online SDP problem can be applied to the
online linear optimization over vectors.

\section{FTRL and its regret bounds by standard derivations}

Follow the Regularized Leader (FTRL) is a standard framework for
designing algorithms for a wide class of online optimizations
(See, e.g., \cite{Shalev-Shwartz2012}).
To employ the FTRL, we need to specify a convex function
$R: \calK \to \realR$ called the regularization function,
or simply the regularizer.
For the online SDP problem $(\calK,\calL)$,
the FTRL with regularizer $R$ suggests
to choose a matrix $\matX_t \in \calK$ as the decision at each round $t$
according to
\[
	\matX_t = \arg \min_{\matX \in \calK}
		\Bigl( R(\matX) + \eta \sum_{s=1}^{t-1} \matL_s \fp \matX \Bigr),
\]
where $\eta > 0$ is a constant called the learning rate.
Throughout the paper, we assume for simplicity that
all the regularizers $R$ are differentiable.

The next lemma gives a general method of deriving regret bounds.

\begin{lemm}[See, e.g., Theorem 2.11 of \cite{Shalev-Shwartz2012}] 
\label{lem:standardDerivation}
Assume that for some real numbers $s,g > 0$ and a norm $\| \cdot \|$
the following holds.
\begin{enumerate}
\item
$R$ is $s$-strongly convex with respect to the norm $\| \cdot \|$, i.e.,
for any $\matX, \matY \in \calK$,
\[
		R(\matX) \geq R(\matY) + \nabla R(\matX) \fp (\matX-\matY)
			+ \frac{s}{2} \| \matX - \matY \|^2,
\]
or equivalently, for any $\matX \in \calK$ and
$\matW \in \mbb{R}^\tt{N}$,
\[
	\vectorize(\matW)^\T \nabla^2 R(\matX) \vectorize(\matW)
		\geq s \|\matW\|^2.
\]
\item
Any loss matrix $\matL \in \calL$ satisfies $\| \matL \|_* \leq g$,
where $\|\cdot\|_*$ is the dual norm of $\| \cdot \|$.
\end{enumerate}
Then, the FTRL with regularizer $R$ achieves
\[
	Reg(T,\calK,\calL) \leq 2 g
		\sqrt{\frac{\max_{\matX,\matX' \in \calK} (R(\matX)
			- R(\matX'))}{s}T}
\]
for an appropriate choice of the learning rate $\eta$. 
\end{lemm}

In the subsequent subsections, we give regret bounds
for the FTRL with popular regularizers.
The first two are straightforward to derive
from known results.

\subsection{FTRL with the Frobenius norm regularization}

The Frobenius norm regularization function is defined as
$R(\matX) = \frac{1}{2} \nFr{\matX}^2$, 
which is the matrix analogue of the $L_2$-norm for
vectors. The FTRL with this regularizer yields the
online gradient descent (OGD) algorithm~\cite{Haz2016}.
Since $R$ is 1-strongly convex with respect to $\nFr{\cdot}$
and the dual of $\nFr{\cdot}$ is $\nFr{\cdot}$,
Lemma~\ref{lem:standardDerivation} gives
\begin{equation}
Reg(T,\calK_2,\calL_2) \leq \rho \gamma_2 \sqrt{2T},
\label{eqn:generalOLOFrobenius}
\end{equation}
where
$\calK_2 = \{\matX \in \realSNNPSD:\nFr{\matX} \leq \rho \}$ and  
$\calL_2 = \{\matL \in \realSNN : \nFr{\matL} \leq \gamma_2 \}$.

\subsection{FTRL with the entropic regularization}

The entropic regularization function is defined as
$R(\matX) = \Tr(\matX \log \matX - \matX)$, 
which is the matrix analogue of the unnormalized entropy
for vectors.
Slightly modifying the proof in \cite{HazanKaleShalev-Shwartz2012},
we obtain the following regret bound for the FTRL with this
regularizer:
\begin{equation}
Reg(T,\calK_1,\calL_\I) \leq 2 \tau \gamma_\I \sqrt{T \log N},
\end{equation}
where
$\calK_1 = \{\matX \in \realSNNPSD:\nTr{\matX} \leq \tau \}$ and
$\calL_\I = \{\matL \in \realSNN : \nSp{\matL} \leq \gamma_\I \}$.

\subsection{FTRL with the log-determinant regularization}

The log-determinant regularization function is defined as
$R(\matX) = -\ln \det(\matX + \eps \matE)$
where $\eps$ is a positive constant.
This is the matrix analogue of the Burg entropy
$-\sum_{i=1}^N \ln x_i$ for vectors $\vecx$ whose
Bregman divergence is the Itakura-Saito divergence. 
The constant $\epsilon $ stabilizes the regularizer to
make the regret bound finite.
Unfortunately, it is unclear what norm is appropriate for
measuring the strong convexity of the log-determinant
regularizer to obtain a tight regret bound.
In the next theorem, we examine the spectral norm and
give a (probably loose) regret bound for the online SDP
problem $(\calK_\I, \calL_1)$, where
$\calK_\I = \{\matX \in \realSNNPSD:\nSp{\matX} \leq \sigma \}$ and 
$\calL_1 = \{\matL \in \realSNN : \nTr{\matL} \leq \gamma_1 \}$.

\begin{theo}
\label{thm:generalOLOLogDet}
The FTRL with the log-determinant regularizer with
$\eps = \sigma$ achieves
\begin{equation}
\label{eqn:generalOLOLogDet}
	Reg(T,\calK_\I,\calL_1) \leq 4 \sigma \gamma_1 \sqrt{T N \ln 2}.
\end{equation}
\end{theo}

\begin{proof}
Below we show that $R$ is $(1/(4\sigma^2))$-strongly convex
with respect to $\nSp{\cdot}$ and
$R(\matX) - R(\matX') \leq N\ln 2$ for any $\matX,\matX' \in \calK$.
Since $\nTr{\cdot}$ is the dual norm of $\nSp{\cdot}$ and
it is clear that $\nTr{L} \leq \gamma_1$ for any $\matL \in \calL_1$,
the theorem follows from Lemma~\ref{lem:standardDerivation}.

The strong convexity of the log-determinant can be verified by
showing positive definiteness of the Hessian of $R$.
The Hessian of $R(\matX) = -\ln\det(\matX+\eps\matE)$ is 
$\nabla^2 R(\matX) = (\matX + \eps \matE)^{-1} \kp (\matX + \eps \matE)^{-1}$
where $\kp$ denotes the Kronecker
product \cite{ForthHovlandPhippsUtkeWalther2012}. 
Since an eignvalue of $A \kp B$ is the product of some eigenvalues of
$A$ and $B$ (see, e.g., \cite{cookbook}) and an eigenvalue of $A^{-1}$ is
the reciprocal of an eigenvalue of $A$,
the minimum eigenvalue of $\nabla^2 R(\matX)$ is
$(\nSp{\matX}+\epsilon)^{-2}$. This implies that
$\nabla^2 R(\matX) - (\sigma+\epsilon)^{-2}\matE$ is positive
semi-definite. In other words,
for any $\matW \in \mbb{R}^\tt{N}$,
\[
	\vectorize(\matW)^\T (\nabla^2 R(\matX)
		- (\sigma+\epsilon)^{-2}\matE) \vectorize(\matW) \ge 0.
\]
Rearranging this inequality and using the fact that
$\vectorize(\matW)^\T\vectorize(\matW) = \nFr{\matW}^2 \ge \nSp{\matW}^2$, 
we get $\vectorize(\matW)^\T \nabla^2 R(\matX) \vectorize(\matW)
\ge (\sigma+\epsilon)^{-2} \nSp{\matW}^2$. 
This implies that $R$ is $(1/(4\sigma^2))$-strongly convex
with respect to $\nSp{\cdot}$.

Next we give upper and lower bounds of $R$.
Note that $\det (\matX+\eps\matE)$ is the product of all eigenvalues of
$\matX+\eps\matE$.
Since, all the eigenvalues are positive and the maximum of them
is bounded by $\sigma+\eps$, 
we have $\eps^N \leq \det (\matX+\eps\matE)
\leq (\sigma+\eps)^N = (2\eps)^N$.
So,
$\max_{\matX,\matX' \in \calK}(R(\matX) - R(\matX'))
\leq N\ln 2$.
\end{proof}
Note that this result is not very impressive, because
$\calK_\I \subseteq \calK_2$ with $\rho = \sqrt{N} \sigma$
and $\calL_1 \subseteq \calL_2$ with $\gamma_2 = \gamma_1$,
and hence the FTRL with the Frobenius norm regularizer has
a slightly better regret bound for $(\calK_\I, \calL_1)$.

In the following sections, we consider a special class of
online SDP problems $(\calK,\calL)$ where
$\calK$ and $\calL$ are further restricted by some complicated way.
For such problems, it is unlikely to derive tight regret bounds
from Lemma~\ref{lem:standardDerivation}.

\section{Online matrix prediction and reduction to online SDP}

Before going to our main contribution, we give a more natural
setting to describe various applications, which is called
the online matrix prediction (OMP) problem.
Then we briefly review the result of Hazan et al., saying that
OMP problems are reduced to online SDP problems
$(\calK,\calL)$ of special form~\cite{HazanKaleShalev-Shwartz2012}.
In particular, the loss matrices in $\calL$ obtained by the
reduction are sparse.
This result motivates us to improve regret bounds for
online sparse SDP problems.

An OMP problem is specified by a pair
$(\calW,G)$, where
$\calW \subseteq [-1,1]^{m \times n} $ is a convex set
of matrices of size $m \times n$ and $G > 0$ is a positive real number.
Note that we do not require $m = n$ or $W^\T = W$.
The OMP problem $(\calW,G)$ is described as the following protocol:
In each round $t = 1, 2, \ldots, T$, the algorithm 
\begin{enumerate}
  \item receives a pair $(i_t,j_t) \in [m] \times [n]$ from the adversary,
  \item chooses $\matW_t \in \calW$ and output $W_{t,(i_t,j_t)}$,
  \item receives $G$-Lipschitz convex loss function $\loss_t : [-1,1] \to \realR$ from the adversary, and
  \item suffers the loss $\loss_t(W_{t,(i_t,j_t)})$. 
\end{enumerate}
The goal is to minimize the following regret:
\[
	Reg_\text{OMP}(T,\calW) = \Tsum \loss_t(W_{t,i_t,j_t})
		- \min_{\matU \in \calW} \Tsum \loss_t(U_{i_t,j_t}).
\]

The online max-cut, the online gambling and the online collaborative
filtering problems are instances of the OMP problems.

\bigskip

\noindent
\textbf{Online max-cut:} \quad
On each round, the algorithm receives a pair of nodes
$(i,j) \in [n] \times [n]$ and should decide whether there is an
edge between the nodes. Formally, the algorithm chooses
$\hat{y}_t \in [-1,1]$, which is interpreted as a randomized prediction
in $\{-1,1\}$: predicts 1 with probability $(1 + \hat{y}_t)/2$
and $-1$ with the remaining probability. 
The adversary then gives the true outcome $y_t \in \{-1,1\}$ indicating
whether $(i_t,j_t)$ is actually joined by an edge or not.
The loss suffered by the algorithm is 
$\ell_t(\hat{y}_t) = |\hat{y}_t - y_t|/2$, which is interpreted
as the probability that the prediction is incorrect.
Note that $\ell_t$ is $(1/2)$-Lipschitz.
The decision space $\calW$ of this problem is the convex hull of
the set $\calC$ of matrices that represent cuts, that is,
$\calC = \{\matC^A \in \{-1,1\}^{n \times n} : A \subseteq [n]\}$,
where $\matC^A_{i,j} = 1$ if $((i\in A) \text{ and } (j \notin A))$
or $((i\notin A) \text{ and } (j\in A))$, and 
$\matC^A_{i,j} = -1$ otherwise.
Note that the best offline matrix
$\matC^A = \arg\min_{\matC^A \in \calC} \sum_t \ell_t(U_{i_t,j_t})$
in $\calC$ is the matrix corresponding to the max-cut $A$
in the weighted graph whose edge weight are
given by $w_{ij} = \sum_{t:(i_t,j_t)=(i,j)} y_t$
for every $(i,j)$~\cite{HazanKaleShalev-Shwartz2012}.
This is the reason why the problem is called online max-cut.

\medskip

\noindent
\textbf{Online gambling:} \quad
On each round, the algorithm receives a pair of teams
$(i,j) \in [n] \times [n]$, and should decide
whether $i$ is going to beat $j$ or not in the upcoming game. 
The decision space is the convex hull of all permutation
matrices $\matW^P \in \{-1,1\}^{n \times n}$,
where $\matW^P$ is the matrix corresponding to permutation $P$ over $[n]$
that satisfies
$W^P_{i,j} = 1$ if $i$ appears before $j$ in the permutation $P$
and $W^P_{i,j} = -1$ otherwise.

\medskip

\noindent
\textbf{Online collaborative filtering:} \quad
We described this problem in Introduction. 
We consider
$\calW = \{\matW \in [-1,1]^{n \times m}:
\nTr{\matW} \leq \tau\}$ for some constant $\tau > 0$,
which is a typical choice for the decision space in the literature.

\bigskip

The next proposition shows how the OMP problem $(\calW, G)$ is reduced to
the online SDP problem $(\calK,\calL)$.
Before stating the proposition, we need to define the notion of
$(\beta,\tau)$-decomposablity of $\calW$.

For a matrix $\matW \in \calW$, let
$\sym(\matW) =
\begin{bmatrix}
       0 & \matW \\
\matW^\T & 0
\end{bmatrix}
$ if $\calW$ is not symmetric (some $\matW \in \calW$ is
not symmetric) and
$\sym(\matW) = \matW$ otherwise. Let $p$ be the order of
$\sym(\matW)$, that is, $p=m+n$ if $\calW$ is not symmetric
and $p = n$ otherwise.
Note that any symmetric matrix can be represented by the difference
of two symmetric and positive semi-definite matrices.
For real numbers $\beta>0$ and $\tau > 0$,
the decision space $\calW$ is $(\beta,\tau)$-decomposable if
for any $\matW \in \calW$, there exists 
$\matP, \matQ \in \realS_+^{\tt{p}}$ such that
$\sym(\matW) = \matP - \matQ$, $\Tr(\matP+\matQ) \leq \tau$ and
$\matP_{i,i} \leq \beta$, $\matQ_{i,i} \leq \beta$ for every $i \in [p]$.

\begin{prop}[Hazan et al.~\cite{HazanKaleShalev-Shwartz2012}]
\label{prop:Reduction}
Let $(\calW,G)$ be the OMP problem where
$\calW \subseteq [-1,1]^{m \times n}$ is $(\beta,\tau)$-decmoposable.
Then, the OMP problem $(\calW,G)$
can be reduced to the online SDP problem
$(\calK,\calL)$, where $N = 2(m+n)$ if $\calW$ is not symmetric
and $N=m=n$ otherwise, and
\begin{eqnarray*}
	\calK & = & \{ \matX \in \realSNNPSD :
		\nTr{\matX}\le \tau, \forall i\in[N], X_{i,i} \le \beta \}, \\
	\calL & = & \{ \matL \in \realSNN : 
		\forall (i,j) \in [N]\times[N], L_{i,j} \leq G, \\
	& & \qquad\qquad\qquad |\{(i,j) : L_{i,j} \neq 0\}| \leq 4, \\
	& & \qquad\qquad\qquad \text{$L^2$ is diagonal}\}.
\end{eqnarray*}
Moreover, the regret of the OMP problem is bounded by
that of the reduced online SDP problem
\[
	Reg_\text{OMP}(T,\calW) \leq \frac{1}{2}Reg(T,\calK,\calL).
\]
\end{prop}

Note that the loss space $\calL$ obtained by the reduction
is very sparse: each loss matrix has only 4 non-zero entries.
Thus, we can say that for every $\matL \in \calL$,
$\nFr{\matL} \leq 2G$ and $\n{\vectorize(\matL)}_1 \leq 4G$. 

Hazan et al.~also give a regret bound of the FTRL with
the entropic regularizer when applied to the online SDP problem
$(\calK,\calL)$ for $\calK$ obtained by the reduction above with
a larger loss space $\calL$ (thus applicable to the online OMP problems).

\begin{theo}[Hazan et al.~\cite{HazanKaleShalev-Shwartz2012}]
For the online SDP problem $(\calK, \calL)$ where
\label{thm:HazanKaleShalev-Shwartz2012}
\begin{eqnarray*}
\calK&=&\{\matX \in \realSNNPSD: \nTr{\matX} \leq \tau,
\forall i \in [N], X_{i,i} \leq \beta \}, \\
\calL&=&\{\matL \in \realSNN: \Tr(\matL^2) \leq \gamma,
\text{$\matL^2$ is diagonal} \},
\end{eqnarray*}
the FTRL with the entropic regularizer
$R(\matX) = \Tr(\matX \ln \matX - \matX)$ achieves 
\[
	Reg(T,\calK,\calL) \leq 2\sqrt{\beta \tau \gamma T \ln N}. 
\]
\end{theo}

Combining Proposition~\ref{prop:Reduction} and
Theorem~\ref{thm:HazanKaleShalev-Shwartz2012}, we can easily get
regret bounds for OMP problems.

\begin{coro}
\label{cor:OMPboundByEnt}
For the OMP problem $(\calW,G)$, where
$\calW \subseteq [-1,1]^{m \times n}$ is $(\beta,\tau)$-decomposable,
there exists an algorithm that achieves
\[
	Reg_\text{OMP}(T,\calK,\calL) = O(G \sqrt{\beta \tau T \ln(m+n)}). 
\]
\end{coro}

Hazan et al.\ apply the bound to the three applications, for which
the decision classes $\calW$ are all $(\beta,\tau)$-docomposable
for some $\beta$ and $\tau$~\cite{HazanKaleShalev-Shwartz2012}.
More specifically, the results are summarized as shown below.
\begin{description}
\item[Online max-cut:]
The problem is
$(1,n)$-decomposable and thus has a regret bound of
$O(G\sqrt{nT\ln n})$.

\item[Online gambling:]
The problem is
$(O(\ln n), O(n\ln n))$-decomposable and thus
has a regret bound of $O(G \sqrt{nT(\ln n)^3})$.

\item[Online collaborative filtering:]
The problem
is $(\sqrt{m+n},2\tau)$-decomposable and thus has a regret bound of
$O(G\sqrt{\tau T \sqrt{m+n} \ln(m+n)})$, which is
$O(G\sqrt{\tau T \sqrt{n} \ln n})$ if we assume without
loss of generality that $n \geq m$.

\end{description}

Christiano provides another technique of reduction from a special
type of OMP problems to a spcial type of online SDP problems,
and apply the FTRL with the log-determinant
regularizer~\cite{Christiano2014}.
He then improves the regret bound for the online max-cut problem
to $O(G\sqrt{nT})$, which matches a lower bound up to a constant factor.
However, the regret bound for online gambling is much worse
($O(Gn^2\sqrt{T})$) and his reduction cannot be applied to online
collaborative filtering.
It is worth noted that the loss matrices obtained by his reduction
are not just sparse but \emph{block-wise sparse}, by which we mean
non-zero entries forming at most two block matrices, and
seemingly his regret analysis depends on this fact.

\section{Main results}

Motivated by the sparse online SDP problem reduced from an OMP problem,
we consider a specific problem $(\tilde \calK,\tilde \calL)$, where
\begin{eqnarray*}
	\tilde \calK & = & \{\matX \in \realSNNPSD:
		\nTr{\matX} \le \tau, \forall i \in[N], X_{i,i} \le \beta \}, \\
	\tilde \calL & = & \{\matL \in \realSNN: \n{\vectorize(\matL)}_1 \le g_1 \}, 
\end{eqnarray*}
and give a regret bound of the FTRL with the log-determinant
regularizer. Note that $\tilde \calK$ is the same as the one obtained
by the reduction and $\tilde \calL$ is much larger if $g_1 = 4G$.
By Proposition~\ref{prop:Reduction}
the regret bound immediately yields a regret bound for the OMP problem
$(\calW,G)$ for a $(\beta,\tau)$-decomposable decision space $\calW$,
which turns out to be tigher than the one using 
the entropic regularizer shown in
Theorem~\ref{thm:HazanKaleShalev-Shwartz2012}.

Our analysis partly follows that of \cite{Christiano2014}
with some generalizations.
In particular, we figure out a general method for deriving
regret bounds by using a new notion of strong convexity
of regularizers, which is implicitly used in \cite{Christiano2014}.
First we state the general theory.

\subsection{A general theory}

We begin with an intermediate bound known as
the FTL-BTL (Follow-The-Leader-Be-The-Leader) Lemma. 
\begin{lemm}[Hazan \cite{Hazan2009}]
\label{lem:FTLBTLLemma}
The FTRL with the regularizer $R:\calK \to \realR$ for
an online SDP problem $(\calK,\calL)$ achieves
\begin{equation}
	Reg(T,\calK,\calL) \leq \frac{H_0}{\eta}
		+ \Tsum \matL_t \fp (\matX_t-\matX_{t+1}), 
\label{eqn:FTLBTLLemma}
\end{equation}
where $H_0 = \max_{\matX,\matX' \in \calK} (R(\matX)-R(\matX')$. 
\end{lemm}
Thanks to this lemma, all we need to do is to bound $H_0$ and
$\matL_t \fp (\matX_t-\matX_{t+1})$.

Now we define the new notion of strong convexity.
Intuitively, this is an integration of
the strong convexity of regularizers with respect to a norm and
the Lipschitzness of loss functions with respect to the norm..
\begin{defi}
For a decision space $\calK$ and a real number $s > 0$,
a regularizer $R:\calK \to \realR$ is said to be $s$-strongly convex
with respect to a loss space $\calL$
if for any $\alpha \in [0,1]$, any $\matX,\matY \in \calK$,
and any $\matL \in \calL$,
\begin{equation}
\begin{split}
R(\alpha\matX&+(1-\alpha)\matY) \\
& \leq \alpha R(\matX)+(1-\alpha)R(\matY)
         -\frac{s}{2}\alpha (1-\alpha) | \matL\fp(\matX-\matY) |^2.
\end{split}
\label{eqn:defstrcvx1}
\end{equation}
\end{defi}
The condition (\ref{eqn:defstrcvx1}) is equivalent to the following
condition~\cite{Nesterov2004}:
for any $\matX, \matY \in \calK$ and $\matL \in \calL$, 
\begin{equation}
R(\matX) \geq R(\matY) +
\nabla R(\matY)\fp(\matX-\matY)+\frac{s}{2}|\matL\fp(\matX-\matY)|^2.
\label{eqn:defstrcvx2}
\end{equation}
Note that the condition (\ref{eqn:defstrcvx2}) has the same form as
the condition of $s$-strong convexity given in
Lemma~\ref{lem:standardDerivation} except
$\|\matX - \matY\|$ is replaced by
$|\matL\fp(\matX-\matY)|$.

The following lemma gives a bound of the term 
$\matL_t \fp (\matX_t-\matX_{t+1}) $ in inequality \eqref{eqn:FTLBTLLemma}
in terms of the strong convexity of the regularizer.
The lemma is implicitly stated in~\cite{Shalev-Shwartz2012}
and hence is not essentially new.
But we give a proof for completeness since it is not very
straighforward to show.

\begin{lemm}[Main lemma]
\label{lem:FTRLWithStronglyConvexAndLipschitzRegularizer}
Let $R:\calK \to \realR$ be $s$-strongly convex 
with respect to $\calL$ for $\calK$. 
Then, the FTRL with the regularizer $R$ applied to
$(\calK,\calL)$ achieves 
\[
	Reg(T,\calK,\calL) \leq 2\sqrt{\frac{H_0}{s}}
\]
for an appropriate choice of $\eta$.
\end{lemm}

\begin{proof}
By Lemma~\ref{lem:FTLBTLLemma}, it suffices to show that
\[
	\matL_t \fp (\matX_t-\matX_{t+1}) \leq \frac{\eta}{s},
\]
since the theorem follows by setting $\eta = \sqrt{sH_0/T}$.
In what follows, we prove the inequality.
First observe that
any $s$-strongly convex function $F$ with respect to
$\calL$ satisfies
\begin{equation}
\label{eqn:lemma2.8on[8]}
F(\matX)-F(\matY) \ge \frac{s}{2} |\matL \fp (\matX-\matY)|^2
\end{equation}
for any $\matX \in \calK$ and any $\matL \in \calL$ for
$\matY = \arg\min_{\matX \in \calK} F(\matX)$.
To see this, we use (\ref{eqn:defstrcvx2}) 
(with replacement of $R$ by $F$) due to the
strong convexity of $F$ and
$\nabla F(\matY) \fp (\matX - \matY) \geq 0$
(otherwise $\matY$ would not be the minimizer since
we can make a small step in the direction
$\matX - \matY$ and decrease the value of $F$.)
See the proof of Lemma 2.8 of \cite{Shalev-Shwartz2012} 
for more detail.

Recall that the update rule of the FTRL is
$\matX_{t+1} = \arg\min_{\matX \in \calK} F_t(\matX)$
where $F_t(\matX) = \sum_{i=1}^t \eta \matL_i \fp \matX + R(\matX)$. 
Note that $F_t$ is $s$-strongly convex with respect to $\calL$
due to the linearity of $\matL_i \fp \matX$. 
Applying \eqref{eqn:lemma2.8on[8]} to $F_t$ and $F_{t-1}$ with
$\matL = \matL_t$, we get 
\begin{eqnarray*}
F_t(\matX_t) &\ge& F_t(\matX_{t+1}) + \frac{s}{2} |\matL_t \fp (\matX_t-\matX_{t+1})|^2, \\
F_{t-1}(\matX_{t+1}) &\ge& F_{t-1}(\matX_t) + \frac{s}{2} |\matL_t \fp (\matX_{t+1}-\matX_t)|^2.
\end{eqnarray*}
Summing up these two inequalities we get 
$$ \eta \matL_t \fp (\matX_t-\matX_{t+1}) \ge s |\matL_t \fp (\matX_t-\matX_{t+1})|^2.$$
Dividing both side by $\matL_t \fp (\matX_t-\matX_{t+1})$ we get the desired result. 
\end{proof}

Note that Lemma~\ref{lem:FTRLWithStronglyConvexAndLipschitzRegularizer}
gives a more general method of deriving regret bounds
than the standard one given by Lemma~\ref{lem:standardDerivation}.
To see this, assume that the two conditions of
Lemma~\ref{lem:standardDerivation} hold.
Then, Cauchy-Schwarz inequality says that
$|\matL \fp (\matX - \matY)| \leq \|\matL\|_* \|\matX - \matY\|
\leq g\|\matX - \matY\|$
for every $\matL \in \calL$ and $\matX,\matY \in \calK$,
where the second inequality is from the second condition.
Thus, the first condition implies the condition of
Lemma~\ref{lem:FTRLWithStronglyConvexAndLipschitzRegularizer}
with $s$ replaced by $s/g^2$ as
\begin{eqnarray*}
	R(\matX) & \geq & R(\matY) + \nabla R\fp(\matX - \matY)
		+ \frac{s}{2}\|\matX - \matY\|^2 \\
	& \geq & R(\matY) + \nabla R\fp(\matX - \matY)
		+ \frac{s}{2g^2}|\matL \fp (\matX - \matY)^2|.
\end{eqnarray*}
Another advantage of using 
Lemma~\ref{lem:FTRLWithStronglyConvexAndLipschitzRegularizer}
is that we can avoid looking for appropriate norms
to obtain good regret bounds.

\subsection{Strong convexity of the log-determinant regularizer}

Now we prove the strong convexity of the log-determinant 
for our problem $(\tilde \calK,\tilde \calL)$ defined in the beginning of
this section.
The following lemma provides a sufficient condition that
turns out to be useful.

\begin{lemm}[Christiano \cite{Christiano2014}]
\label{lem:strongConvexityOfLogDeterminant}
Let $\matX, \matY \in \realSNNPSD$ be such that
\[
	\exists (i,j) \in \tt{[N]},
		|X_{i,j}-Y_{i,j}| \ge \delta (X_{i,i}+X_{j,j}+Y_{i,i}+Y_{j,j}). 
\]
Then the following inequality holds:
\begin{eqnarray*}
\begin{aligned}
 -\ln\det(\alpha\matX &+ (1-\alpha) \matY) \\
 &\leq
 -\alpha \ln \det(\matX) - (1-\alpha) \ln \det(\matY)
 -\frac{\alpha(1-\alpha)}{2}\frac{\delta^2}{72 \sqrt{e}}.
\end{aligned}
\end{eqnarray*}
\end{lemm}
The proof of this lemma is given in Appendix. 
Note that the original proof by Christiano only gives 
the order of the lower bound of the last term of $\Omega(\delta^2)$. 
So we give the proof with a constant factor.

The next lemma shows that the sufficient condition actually holds
for our problem $(\tilde\calK,\tilde\calL)$ for
$\delta = O(|\matL \fp (\matX-\matY)|)$, which establishes
the strong convexity of the log-determinant regularizer.
The lemma is a slight generalization of
\cite{Christiano2014} in that loss matrices are not necessarily
block-wise sparse.

\begin{lemm}
\label{lem:strongConvexityWithRespectToLinear}
Let $\matX,\matY \in \realSNN$ be such that
$X_{i,i} \leq \beta'$ and $Y_{i,i} \leq \beta'$ for every $i \in [N]$.
Then, for any $\matL \in \tilde \calL$,
there exists $(i,j) \in \tt{[N]}$ such that
\[
	|X_{i,j} - Y_{i,j}| \geq
		\frac{|\matL\fp(\matX - \matY)|}{4 g_1 \beta'}
		(X_{i,i} + X_{j,j} + Y_{i,i} + Y_{j,j}).
\]
\end{lemm}

\begin{proof}
By Cauchy-Schwarz inequality, 
\begin{eqnarray*}
	|\matL \fp (\matX-\matY)|
	& \leq & \n{\vectorize(\matL)}_1 \n{\vectorize(\matX-\matY)}_\I
	\leq g_1 \max_{i,j} |X_{i,j}-Y_{i,j}|.
\end{eqnarray*}
Thus the lemma follows since
$X_{i,i} + X_{j,j} + Y_{i,i} + Y_{j,j} \leq 4\beta'$.
\end{proof}
Applying Lemma~\ref{lem:strongConvexityWithRespectToLinear}
to $\matX + \eps E$ and $\matY + \eps E$ for
$\matX, \matY \in \tilde\calK$ and $\beta' = \beta + \eps$,
and then applying
Lemma~\ref{lem:strongConvexityOfLogDeterminant},
we immediately get the following proposition.
\begin{prop}
\label{pro:strongConvexityOfLogDeterminant}
The log-determinant regularizer
$R(\matX)=-\ln \det(\matX+\epsilon \matE)$ is 
$s$-strongly convex with respect to $\tilde\calL$ for $\tilde\calK$
with
$s = 1/(1152 \sqrt{e} g_1^2 (\beta+\epsilon)^2)$.
\end{prop}

Combining this proposition with
Lemma \ref{lem:FTRLWithStronglyConvexAndLipschitzRegularizer},
we can derive a regret bound. 

\begin{theo}[Main theorem]
\label{thm:mainTheorem}
For the online SDP problem $(\tilde\calK,\tilde\calL)$,
the FTRL with the log-determinant regularizer
$R(\matX) = -\ln \det(\matX + \epsilon \matE)$ 
achieves
\[
	Reg(T,\calK,\calL) \leq 175 g_1 \sqrt{\beta \tau T}
\]
for appropriate chioces of $\eta$ and $\epsilon$.
\end{theo}

\begin{proof}
We know that $R$ is $s$-strongly convex
for $s = 1/(1152 \sqrt{e} g_1^2 (\beta+\epsilon)^2)$ 
by Proposition~\ref{pro:strongConvexityOfLogDeterminant}.
It remains to give a bound on
$H_0 = R(\matX_0) - R(\matX_1)$, where
$\matX_0$ and $\matX_1$ be the maximizer and the minimizer of $R$
in $\tilde\calK$, respectively.
Let $\lambda_i(\matX)$ be the $i$-th eigenvalue of $\matX$. Then,
\begin{eqnarray*}
	R(\matX_0)-R(\matX_1) &=&
		-\ln \det(\matX_0+\eps E)
			+ \ln \det(\matX_1+\eps E) \\
	& = &
	\sum_{i=1}^N \ln \frac{\lambda_i(\matX_1)
		+ \epsilon}{\lambda_i(\matX_0) + \epsilon}
	\le
	\sum_{i=1}^N \ln \left(
		\frac{\lambda_i(\matX_1)}{\epsilon}+1
	\right) \\
	& \le &
	\sum_{i=1}^N \frac{\lambda_i(\matX_1)}{\epsilon}
	= \frac{\Tr(\matX_1)}{\epsilon}
	= \frac{\nTr{\matX_1}}{\epsilon}
	\leq \frac{\tau}{\epsilon}.
\end{eqnarray*}
Note that we use the inequality $\ln(x+1)\le x $ for $-1 < x$.
Applying
Lemma~\ref{lem:FTRLWithStronglyConvexAndLipschitzRegularizer}
with $\eps=\beta$, we get the theorem.
\end{proof}

Since the OMP problem $(\calW,G)$ for a $(\beta,\tau)$-decomposable
decision space $\calW$ can be reduced to the online SDP problem
$(\tilde\calK,\tilde\calL)$ with $g_1 = 4G$,
Proposition~\ref{prop:Reduction} implies the following regret
bound for the OMP problem.

\begin{coro}
For the OMP problem $(\calW,G)$ where $\calW \subseteq [-1,1]^{m \times n}$
is $(\beta,\tau)$-decomposable, there exists an algorithm that achieves
\[
	Reg_\text{OMP}(T,\calW) = O(G\sqrt{\beta\tau T}).
\]
\end{coro}

Note that the bound does not depend on the size
($m$ or $n$) of matrices and improves by a factor of $O(\sqrt{m+n})$
from Corollary~\ref{cor:OMPboundByEnt}.
Accordingly, we get $O(\sqrt{\ln n})$ improvements
for the three application problems:
\begin{description}
\item[Online max-cut] has a regret bound of $O(G\sqrt{nT})$.
\item[Online gambling] has a regret bound of $O(G\ln n\sqrt{nT})$.
\item[Online collaborative filtering] has a regret bound of
$O(G\sqrt{\tau T \sqrt{n}})$ for $n \geq m$.
\end{description}
All these bounds match the lower bounds
given in~\cite{HazanKaleShalev-Shwartz2012}
up to constant factors.

\subsection{The vector case}

We can apply the results obtained above to the vector case
by just restricting the decision and loss spaces to diagonal matrices.
That is, our problem $(\tilde\calK,\tilde\calL)$ is now rewritten as
\begin{eqnarray*}
	\tilde\calK & = & \{\diag(\vecx):\vecx \in \realR_+^N,
		\n{\vecx}_\I \leq \beta, \n{\vecx}_1 \leq \tau \}, \text{ and} \\
	\tilde\calL & = & \{\diag(\vecl):\vecl \in \realR^N,
		\n{\vecl}_1 \leq g_1\},
\end{eqnarray*}
and the log-determinant is a variant of the Burg entropy
$R(\diag(\vecx)) = -\sum_i^N \ln(x_i+\eps)$.
Applying Theorem~\ref{thm:mainTheorem} to the problem,
we have $O(g_1\sqrt{\beta\tau T})$ regret bound.

Curiously, unlike the matrix case, we can also apply the standard
technique, namely, Theorem~\ref{thm:generalOLOLogDet}
(with a slight modification), to get the same regret bound.
To see this, observe that
$\nSp{\diag(\vecx)} = \n{\vecx}_\I \leq \beta$
for every $\diag(\vecx) \in \tilde\calK$, and
$\nTr{\diag(\vecl)} = \n{\vecl}_1 \leq g_1$
for every $\diag(\vecl) \in \tilde\calL$.
These imply that $\tilde\calK \subseteq \calK_\I$
with $\sigma = \beta$
and $\tilde\calL \subseteq \calL_1$ with $\gamma_1 = g_1$.
Moreover, as shown in the proof of Theorem~\ref{thm:mainTheorem},
we have
$\max_{\matX,\matX'\in\tilde\calK}(R(\matX)-R(\matX'))
\leq \tau/\eps$.
So, $N\ln 2$ in Theorem~\ref{thm:generalOLOLogDet}
can be replaced by $\tau/\eps$, and hence we get
a regret bound of $4g_1\sqrt{\beta\tau T}$.

\section{Conclusion}
In this paper, we consider the online symmetric positive semidefinite matrix prediction problem. 
We proposed a FTRL-based algorithm with the log-determinant regularization.
We tighten and generalize existing analyses. 
As a result, we show that the log-determinant regularizer is effective when loss matrices are sparse. 
Reducing online collaborative filtering task to the online SDP tasks with sparse loss matrices, 
our algorithms obtain optimal regret bounds. 

Our future work includes
(\romannumeral1) imploving a constant factor in the regret bound, 
(\romannumeral2) applying our method to other online prediction tasks with sparse loss settings 
including the ``vector'' case,
(\romannumeral3) developing a fast implementation of our algorithm.

\section*{Acknowledgments}
This work is supported in part  by JSPS KAKENHI Grant Number 16K00305 and JSPS KAKENHI Grant
Number 15H02667, respectively.

\bibliographystyle{ieicetr}
\bibliography{}

\begin{thebibliography}{10}

\bibitem{Degenne2016}
R.~Degenne and V.~Perchet, ``Combinatorial semi-bandit with known covariance,''
  Advances in Neural Information Processing Systems 29, pp.2964--2972, 2016.

\bibitem{Kale2016}
S.~Kale, C.~Lee, and D.~P{\'{a}}l, ``Hardness of online sleeping combinatorial
  optimization problems,'' Advances in Neural Information Processing Systems
  29, pp.2181--2189, 2016.

\bibitem{Cutkosky2017}
A.~Cutkosky and K.A. Boahen, ``Online learning without prior information,''
  Proceedings of the 30th Conference on Learning Theory, {COLT} 2017,
  pp.643--677, 2017.

\bibitem{Cesa-Bianchi2011}
N.~Cesa{-}Bianchi and O.~Shamir, ``Efficient online learning via randomized
  rounding,'' Advances in Neural Information Processing Systems 24,
  ed.~J.~Shawe-Taylor, R.S. Zemel, P.L. Bartlett, F.~Pereira, and K.Q.
  Weinberger, pp.343--351, Curran Associates, Inc., 2011.

\bibitem{Herbster2016}
M.~Herbster, S.~Pasteris, and M.~Pontil, ``Mistake bounds for binary matrix
  completion,'' Advances in Neural Information Processing Systems 29,
  pp.3954--3962, 2016.

\bibitem{Jin2016}
C.~Jin, S.M. Kakade, and P.~Netrapalli, ``Provable efficient online matrix
  completion via non-convex stochastic gradient descent,'' Advances in Neural
  Information Processing Systems 29, pp.4520--4528, 2016.

\bibitem{Srebro2005}
N.~Srebro and A.~Shraibman, ``Rank, trace-norm and max-norm,'' Learning Theory,
  18th Annual Conference on Learning Theory, {COLT} 2005, Bertinoro, Italy,
  June 27-30, 2005, Proceedings, ed.~P.~Auer and R.~Meir, Lecture Notes in
  Computer Science, vol.3559, pp.545--560, Springer, 2005.

\bibitem{Mazumder2010}
R.~Mazumder, T.~Hastie, and R.~Tibshirani, ``Spectral regularization algorithms
  for learning large incomplete matrices,'' J. Mach. Learn. Res., vol.11,
  pp.2287--2322, Aug.\ 2010.

\bibitem{Shamir2011}
O.~Shamir and S.~Shalev{-}Shwartz, ``Collaborative filtering with the trace
  norm: Learning, bounding, and transducing,'' {COLT} 2011 - The 24th Annual
  Conference on Learning Theory, June 9-11, 2011, Budapest, Hungary,
  pp.661--678, 2011.

\bibitem{Koltchinskii2011}
V.~Koltchinskii, K.~Lounici, and A.B. Tsybakov, ``Nuclear-norm penalization and
  optimal rates for noisy low-rank matrix completion,'' The Annals of
  Statistics, vol.39, no.5, pp.2302--2329, 2011.

\bibitem{HazanKaleShalev-Shwartz2012}
E.~Hazan, S.~Kale, and S.~Shalev-Shwartz, ``Near-optimal algorithms for online
  matrix prediction,'' CoRR, vol.abs/1204.0136, 2012.

\bibitem{Hazan2009}
E.~Hazan, ``A survey: The convex optimization approach to regret
  minimization,'' 2009.

\bibitem{RakhlinAbernethyAgarwalBartlettHazanTewari2009}
A.~Rakhlin, J.~Abernethy, A.~Agarwal, P.~Bartlett, E.~Hazan, and A.~Tewari,
  ``Lecture notes on online learning draft,'' 2009.

\bibitem{Shalev-Shwartz2012}
S.~Shalev-Shwartz, ``Online learning and online convex optimization,'' Found.
  Trends Mach. Learn., vol.4, no.2, pp.107--194, Feb.\ 2012.

\bibitem{Haz2016}
E.~Hazan, ``Introduction to online convex optimization,'' Foundations and
  Trends in Optimization, vol.2, no.3-4, pp.157--325, 2016.

\bibitem{TsudaRatschWarmuth2005}
K.~Tsuda, G.~R\"{a}tsch, and M.K. Warmuth, ``Matrix exponentiated gradient
  updates for on-line learning and bregman projection,'' J. Mach. Learn. Res.,
  vol.6, pp.995--1018, Dec.\ 2005.

\bibitem{DavisKulisJainSraDhillon2007}
J.V. Davis, B.~Kulis, P.~Jain, S.~Sra, and I.S. Dhillon,
  ``Information-theoretic metric learning,'' Proceedings of the 24th
  international conference on Machine learning, pp.209--216, ACM, 2007.

\bibitem{RavikumarWainwrightRaskuttiYu2011}
P.~Ravikumar, M.J. Wainwright, G.~Raskutti, and B.~Yu, ``High-dimensional
  covariance estimation by minimizing $\ell_1$-penalized log-determinant
  divergence,'' Electronic Journal of Statistics, vol.5, pp.935--980, 2011.

\bibitem{Christiano2014}
P.~Christiano, ``Online local learning via semidefinite programming,''
  Symposium on Theory of Computing, {STOC} 2014, pp.468--474, 2014.

\bibitem{Dattorro2005}
J.~Dattorro, Convex Optimization \& Euclidean Distance Geometry, Meboo
  Publishing USA, 2005.

\bibitem{ForthHovlandPhippsUtkeWalther2012}
S.~Forth, P.~Hovland, E.~Phipps, J.~Utke, and A.~Walther, Recent Advances in
  Algorithmic Differentiation, Lecture Notes in Computational Science and
  Engineering, Springer, 2012.

\bibitem{cookbook}
K.B. Petersen and M.S. Pedersen, ``The matrix cookbook,'' nov\ 2012.
\newblock Version 20121115.

\bibitem{Nesterov2004}
Y.~Nesterov, Introductory lectures on convex optimization : a basic course,
  Applied optimization, Kluwer Academic Publ., Boston, Dordrecht, London, 2004.

\bibitem{CoverThomas2012}
T.~Cover and J.~Thomas, Elements of Information Theory, Wiley, 2012.

\end{thebibliography}

\appendix*
\section{Proof of 
Lemma \ref{lem:strongConvexityOfLogDeterminant}}
In this appendix we give a proof of 
Lemma \ref{lem:strongConvexityOfLogDeterminant} 
by showing a series of definitions and technical lemmas.

\renewcommand{\thesection}{\Alph{section}}

The negative entropy function
over the set of probability distributions $P$ over $\realR^N$
is defined as $H(P) = \mbb{E}_{\vecx \sim P}[-\ln P(\vecx)]$.
The total variation distance between probability distributions
$P$ and $Q$ over $\realR^N$ is defined as
$\frac{1}{2}\int_{\vecx} | P(\vecx) - Q(\vecx) | d\vecx$. 
The characteristic function of a probability distribution
$P$ over $\realR^N$ is defined as
$\phi(\vecu) = \expE_{\vecx \simeq P}[ e^{i \vecu^\T \vecx} ]$,
where $i$ is the imaginary unit.

The following lemma shows that the difference of the
characteristic functions gives a lower bound of the total variation
distance.

\begin{lemm}
\label{lem:lowerBoundOfTotalVariationDistance}
Let $P$ and $Q$ be probability distribution over $\realR^{N}$ 
and $\phi_P(\vecu)$, $\phi_Q(\vecu)$ be their characteristic
functions, respectively.  Then, 
\[
	\max_{\vecu \in \realR^N} |\phi_P(\vecu)-\phi_Q(\vecu)|
		\leq \int_{\vecx} |P(\vecx)-Q(\vecx)| d \vecx.
\]
\end{lemm}
\begin{proof}
\begin{eqnarray*}
\max_{\vecu \in \realR^N} |\phi_P(\vecu) - \phi_Q(\vecu)|
&=&\max_{\vecu \in \realR^N} \Bigl| \int_{\vecx} e^{i \vecu^T \vecx} P(\vecx) d\vecx - \int_{\vecx} e^{i \vecu^T \vecx} Q(\vecx) d\vecx \Bigr| \\
&\leq& \max_{\vecu \in \realR^N} \int_{\vecx} |e^{i\vecu^T\vecx}| |P(\vecx)-Q(\vecx)| d\vecx \\
&\leq& \int_{\vecx} |P(\vecx)-Q(\vecx)| d\vecx 
\end{eqnarray*}
where we use the fact that $|e^{i\vecu^T\vecx}|=1$ 
for every $\vecu \in \realRN$.
\end{proof}

The negative entropy function is strongly convex with respect to
the total variation distance.

\begin{lemm}[Christiano \cite{Christiano2014}]
\label{lem:strongConvexityOfNegativeEntropy}
Let $P$ and $Q$ be probability distributions over $\realRN$
with total variation distance $\delta$.
Then,
\[
H(\alpha P + (1-\alpha) Q) \leq 
\alpha H(P) + (1 - \alpha) H(Q) + \alpha(1 - \alpha)\delta^2.
\]
\end{lemm}
In \cite{Christiano2014}, the proof was given for only discrete entropies
and the differential entropies are regarded as the limit of
the discrete entropies, but this assertion is
incorrect~\cite{CoverThomas2012}.
We fix the problem by considering the limit of the ``difference"
of discrete entropies as described below.
First we fix a discretization interval $\Delta$.
As in Sec 8.3 of \cite{CoverThomas2012}, for a continuous
distribution $P$, we define its discretized distribution $P^\Delta$,
and thus we can define the discrete entropy $H(P^\Delta)$. 
Then we have $H(P^\Delta) = H(P) + \ln \Delta$,
and thus for two continuous distributions $P$ and $Q$,
$\lim_{\Delta \to 0} \bigl(H(P^\Delta) - H(Q^\Delta) \bigr)$
converges $H(P) - H(Q)$. Using this, we can prove this lemma.

The following lemma connects the entropy and the log-determinant.

\begin{lemm}[Cover and Thomas \cite{CoverThomas2012}]
\label{lem:upperBoundOfEntropy}
For any probability distribution $P$ over $\realR^N$ with $0$ mean
and covariance matrix $\matSigma$,
its entropy is bounded by the log-determinant of covariance matrix.
That is,
\[
	H(P) \leq \frac{1}{2} \ln (\det(\matSigma)(2 \pi e)^N),
\]
where the equality holds if and only if $P$ is a Gaussian.
\end{lemm}

We need the following technical lemma.
\begin{lemm}
\label{lem:exponentialDiffInequalityImproved}
$ e^{-\frac{x}{2}} - e^{-\frac{1-x}{2}} \geq \frac{e^{-1/4}}{2}(1-2x)$
for $ 0 \leq x \leq 1/2$
\end{lemm}
\begin{proof}
Since the function $f(x) = e^{-x/2} - e^{-(1-x)/2}$ is convex
on $ 0 \le x \le 1/2$, its tangent at $x=1/2$ always gives
a lower bound of $f(x)$.
Hence we get $f(x) \ge f'(1/2) (x-1/2) + f(1/2) = e^{-1/4}(1-2x)/2 $.
\end{proof}

The following lemma provides us a relation between covariance matrices and
the total variation distance. 

\begin{lemm}[Christiano \cite{Christiano2014}]
\label{lem:relationBetweenCovarianceMatrixAndTotalVariationDistance}
Let $\calG_1$ and $\calG_2$ are zero-mean Gaussian distributions
with covariance matrix $\matSigma$ and $\matTheta$, respectively.
If there exists $(i,j) \in [N] \times [N]$ such that
\[
	|\Sigma_{i,j}-\Theta_{i,j}|
		\geq \delta (\Sigma_{i,i} + \Theta_{i,i}
			+ \Sigma_{j,j} + \Theta_{j,j}),
\]
then the total variation distance between $\calG_1$ and $\calG_2$
is at least $\frac{1}{12 e^{1/4}}\delta$.
\end{lemm}
The original proof by Christiano gives an asymptotic
bound of the form of $\Omega(\delta)$. 
Now we give the proof with a constant factor.

\begin{proof}
By Lemma \ref{lem:lowerBoundOfTotalVariationDistance}, 
it is sufficient to derive a lower bound of the maximum of
difference between characteristic functions. 
In this case,
the characteristic functions of $\calG_1$ and $\calG_2$ are
$\phi_1(\vecu) = e^{-\frac{1}{2}\vecu^T \matSigma \vecu}$ and 
$\phi_2(\vecu) = e^{-\frac{1}{2}\vecu^T \matTheta \vecu}$,
respectively.

Let 
$
\alpha_1 = \vecv^T \matSigma \vecv, 
\alpha_2 = \vecv^T \matTheta \vecv, 
\vecu = \frac{\vecv}{\sqrt{\alpha_1 + \alpha_2}}
$
.
Then,
\begin{eqnarray*}
\max_{\vecu \in \realR^N} |\phi_1(\vecu) - \phi_2(\vecu)|
&\geq& \max_{\vecv \in \realR^N} \Bigl|e^{\frac{-\alpha_1}{2(\alpha_1 + \alpha_2)}} - e^{\frac{-\alpha_2}{2(\alpha_1 + \alpha_2)}} \Bigr| \\
&\geq& \max_{\vecv \in \realR^N} \Bigl| \frac{1}{2e^{1/4}}\frac{\alpha_1 - \alpha_2}{\alpha_1 + \alpha_2} \Bigr|.
\end{eqnarray*}      
Note that we use Lemma \ref{lem:exponentialDiffInequalityImproved} 
in the last inequality.

By the assumption, we have for some $(i,j)$ that
\begin{eqnarray*}
\begin{aligned}
	\delta (\Sigma_{i,i} &+ \Theta_{i,i}
		+ \Sigma_{j,j} + \Theta_{j,j}) 
	\leq |\Sigma_{i,j} - \Theta_{i,j}| \\ 
	&= \frac{1}{2}\left|
		(\vece_i+\vece_j)^\T(\matSigma-\matTheta)(\vece_+\vece_j)
			- (\matSigma-\matTheta)_{i,i}
			- (\matSigma-\matTheta)_{j,j}
		\right|
\end{aligned}
\end{eqnarray*}
This implies that one of
$(\vece_i + \vece_j)^\T (\matSigma - \matTheta) (\vece_i + \vece_j)$, 
$\vece_i^\T(\matSigma - \matTheta)\vece_i$, and
$\vece_j^\T(\matSigma - \matTheta)\vece_j$
has absolute value greater than 
$\frac{2\delta}{3} ((\matSigma + \matTheta)_{i,i} + (\matSigma + \matTheta)_{j,j})$.

On the other hand, 
\begin{eqnarray*}
(\vece_i + \vece_j)^\T (\matSigma + \matTheta) (\vece_i + \vece_j) 
& =  &(\matSigma + \matTheta)_{i,i} + (\matSigma + \matTheta)_{j,j} + 2(\matSigma + \matTheta)_{i,j} \\
&\leq& 2 (\matSigma + \matTheta)_{i,i} + 2(\matSigma + \matTheta)_{j,j}. 
\end{eqnarray*}
In the last inequality we use $\matSigma + \matTheta \in \realSNNPSD$ and the fact that 
$X_{i,j} \leq \frac{1}{2}(X_{i,i}+X_{j,j})$ for symmetric semi-definite
matrix $X$.
So,
$$
\forall \vecv \in \{ \vece_i, \vece_j , \vece_i + \vece_j \}, 
\vecv^\T (\matSigma + \matTheta) \vecv \leq 2(\matSigma + \matTheta)_{i,i} + 2(\matSigma + \matTheta)_{j,j}
$$
and thus we have
\begin{eqnarray*}
\max_{\vecu \in \realR^N} |\phi_1(\vecu) - \phi_2(\vecu)|
&\geq& 
\max_{\vecv \in \{\vece_i,\vece_j,\vece_i+\vece_j\}} 
\Bigl| \frac{1}{2e^{1/4}}\frac{\vecv^\T(\matSigma - \matTheta)\vecv}{\vecv^\T(\matSigma + \matTheta)\vecv} \Bigr|
\geq \frac{\delta}{6e^{1/4}}
\end{eqnarray*}
\end{proof}

Now we are ready to give a proof of
Lemma \ref{lem:strongConvexityOfLogDeterminant}. 
\begin{proof}
Let $\calG_1,\calG_2$ are zero-mean Gaussian distributions with 
covariance matrix $\matSigma = \matX, \matTheta = \matY$, respectively.
In the matrix case, by the assumption and Lemma \ref{lem:relationBetweenCovarianceMatrixAndTotalVariationDistance},
total variation distance between $\calG_1$ and $\calG_2$ is at least $\frac{\delta}{12 e^{1/4}}$.
For simplicity of notation, let 
$\tilde{\delta} = \frac{\delta}{12 e^{1/4}}$ in the matrix case
Consider the entropy of the following probability distribution of $\vecv$;
with probability $\alpha$, $\vecv \simeq \calG_1$,
with remaining probability $1-\alpha$, $\vecv \simeq \calG_2$. 
Its covariance matrix is $ \alpha\matSigma + (1-\alpha)\matTheta$. 
By Lemma \ref{lem:strongConvexityOfNegativeEntropy} and \ref{lem:upperBoundOfEntropy},
\begin{eqnarray*}
\ln \det(\alpha\matSigma + (1-\alpha)\matTheta) 
 &\geq& 2 H(\alpha \calG_1 + (1-\alpha)\calG_2) - \ln (2\pi e)^N \\
 &\geq& 2 \alpha H(\calG_1)\!+\!2 (1\!-\!\alpha) H(\calG_2)\!-\!\ln (2\pi e)^N\!+\!\alpha(1\!-\!\alpha)\tilde{\delta}^2 \\
 &=& \alpha \ln \det(\matSigma) + (1 - \alpha)\ln \det(\matTheta) + \alpha(1 - \alpha) \tilde{\delta}^2.
\end{eqnarray*}
\end{proof}

\end{document}